\documentclass[11pt]{article}

\usepackage{fullpage}
\usepackage{natbib} 
\usepackage{mathtools} 
\usepackage{booktabs} 
\usepackage{tikz} 
\usepackage[utf8]{inputenc} 
\usepackage[T1]{fontenc}    
\usepackage{hyperref}       
\usepackage{url}            
\usepackage{booktabs}       
\usepackage{amsfonts}       
\usepackage{nicefrac}       
\usepackage{microtype}      

\usepackage{graphicx}
\usepackage{caption}
\usepackage{subcaption}
\usepackage{times}
\usepackage{mathtools}
\usepackage{algorithm}
\usepackage{algorithmic}

\usepackage{color}
\usepackage{multirow}
\usepackage{amsthm}
\usepackage{mathrsfs}
\usepackage{float}
\usepackage{algorithm}
\usepackage{algorithmic}
\usepackage{wrapfig}

\title{Geometric Rates of Convergence for Kernel-based Sampling Algorithms}

\author{%
  Rajiv Khanna\\
  Department of Statistics\\
  University of California, Berkeley\\
  \texttt{rajivak@berkeley.edu}
  \and
    Liam Hodgkinson \\
  ICSI and Department of Statistics\\
  University of California, Berkeley\\
  \texttt{liam.hodgkinson@berkeley.edu}
  \and
  Michael W. Mahoney\\
  ICSI and Department of Statistics\\
  University of California, Berkeley\\
  \texttt{mmahoney@stat.berkeley.edu}
 }

\newcommand{\wkh}{\text{WKH}}

\newcommand{\bigO}{\mathcal{O}}

\mathchardef\mhyphen="2D

\DeclareMathOperator*{\argmax}{arg\,max}
\DeclareMathOperator*{\argmin}{arg\,min}

\newcommand{\vertiii}[1]{{\left\vert\kern-0.25ex\left\vert\kern-0.25ex\left\vert #1
    \right\vert\kern-0.25ex\right\vert\kern-0.25ex\right\vert}}

\def\bff{{\mathbf{f}}}

\def\bk{{\mathbf{k}}}

\def\bu{{\mathbf{u}}}

\def\bw{{\mathbf{w}}}
\def\bx{{\mathbf{x}}}
\def\by{{\mathbf{y}}}
\def\bz{{\mathbf{z}}}

\def\bK{{\mathbf{K}}}

\def\bT{{\mathbf{T}}}

\def\bX{{\mathbf{X}}}

\def\0{{\mathbf{0}}}

\def\bbE{{\mathbb{E}}}

\def\bbP{{\mathbb{P}}}

\def\bbR{{\mathbb{R}}}

\def\cI{\mathcal{I}}

\def\cP{\mathcal{P}}

\def\cX{\mathcal{X}}
\def\cY{\mathcal{Y}}

\def\sfG{\mathsf{G}}

\def\sfS{\mathsf{S}}
\def\sfT{\mathsf{T}}

\newtheorem{theorem}{Theorem}
\newtheorem{assumption}{Assumption}

\newtheorem{lemma}[theorem]{Lemma}

\renewcommand{\text}[1]{{\textnormal{#1}}}

\begin{document}
\maketitle

\begin{abstract}
The rate of convergence of weighted kernel herding (WKH) and sequential Bayesian quadrature (SBQ), two kernel-based sampling algorithms for estimating integrals with respect to some target probability measure, is investigated. Under verifiable conditions on the chosen kernel and target measure, we establish a near-geometric rate of convergence for target measures that are nearly atomic. Furthermore, we show these algorithms perform comparably to the theoretical best possible sampling algorithm under the maximum mean discrepancy. An analysis is also conducted in a distributed setting. Our theoretical developments are supported by empirical observations on simulated data as well as a real world application.  
\end{abstract}

\section{Introduction} \label{sec:intro}

Estimating expectations is a common problem fundamental to many applications in statistics and machine learning, such as the estimation of sufficient statistics, prediction after marginalization of latent variables, and calculation of risk.
The goal is the computation of integrals of the form 
\begin{equation}
\bbE_\pi f(\bx) = \int_\cX f(\bx)d\pi(\bx) ,
\label{eq:expectation}
\end{equation}
where $f$ is a given function $f$ and $\pi$ a probability measure over $\mathbb{R}^d$. In the vast majority of cases, the integral is not analytically computable, necessitating numerical approximation. If $\pi$ is absolutely continuous with density $p$, classical Gaussian quadrature techniques achieve exponential rates of convergence, but these are known to suffer from the curse of dimensionality.
Instead, one often adopts probabilistic techniques. The simplest of these is crude Monte Carlo (MC) integration: generating independent samples $\{\bx_i\}_{i=1}^m$ from $\pi$, and returning the empirical average of $\{f(\bx_i)\}_{i=1}^m$. This method converges at a prohibitively slow rate of $\bigO(m^{-1/2})$ in $L^2$. If one cannot easily generate independent samples from $\pi$, Markov Chain Monte Carlo (MCMC) is often used instead.  This method generates approximate samples from $\pi$ as iterates of an ergodic Markov chain. Under certain assumptions on the chain, this approach will also converge at rate $\bigO(m^{-1/2})$. 

To achieve ``super-root-$m$'' rates of convergence, alternative non-random sampling techniques have been proposed. Quasi Monte Carlo is a classic example with dimension-dependent $\bigO(m^{-1}\log m)$ convergence. 
More recently, the herding algorithm~\citep{Welling2009a,Welling2009b,Welling2010} was proposed to learn Markov Random Fields (MRFs). First applicable to discrete finite-dimensional spaces, it was later extended to continuous spaces and infinite dimensions through the kernel trick by~\citet{Chen2010SuperSamples}, who called the resulting algorithm \emph{Kernel Herding} (KH). A general convergence rate of $\bigO(m^{-1/2})$ was also provided, showing that KH performs, asymptotically, at least as well as crude Monte Carlo methods. However, in practice, KH typically exhibits faster convergence, suggesting these rates can be improved. 

\citet{Chen2010SuperSamples} also suggested a \emph{moment matching} interpretation of the algorithm: kernel herding is equivalent to choosing 
samples $\{\bx_i\}_{i=1}^m$ 
that minimize 
the Maximum Mean Discrepancy (MMD) metric (see~\citet{Chen2010SuperSamples,Huszar2012OptimallyWeightedHI}) between $\pi$ and the empirical measure of $\{\bx_i\}_{i=1}^m$. In an alternative Bayesian approach,~\citet{Ohagan1991} and~\citet{Ghahramani2002BayesianMC} assume a GP prior on $f$. Each sample $\bx_i$ is generated by minimizing the mean squared error of the posterior of $f$, conditioned on the points $\{\bx_j\}_{j<i}$ already sampled. 
This was shown to be equivalent to KH (with kernel dictated by the covariance of the GP), but with an additional step of also minimizing the weights attached to the sampled points (that is, instead of using uniform weights of $\nicefrac{1}{m}$~\citep{Huszar2012OptimallyWeightedHI}). 
Empirically, it was observed that this new algorithm---called \emph{Sequential Bayesian Quadrature} (SBQ)---converges faster than KH due to the additional weight optimization step. 

A partial explanation for faster convergence  was given by~\citet{Bach2012Herding}. While KH and SBQ both choose the next sample point to minimize MMD, SBQ also optimizes for the weights. 
Alternatively, once a sample point is selected, the weights themselves can be optimized. We refer to this procedure as \emph{Weighted Kernel Herding} (WKH). For the case when the weights are constrained to lie in the unit simplex,~\citet{Bach2012Herding} noted that WKH is equivalent to the classic algorithm of \citet{Frank1956bt} on the marginal polytope. 
Exploiting this connection, they were able to use existing convergence results to analyze constrained WKH. Specifically, if the optimum lies in the relative interior of the marginal polytope with distance $b$ away from the boundary, the convergence rate is $\bigO(e^{-b m})$. 
For infinite dimensional kernels, $b = 0$ (i.e., exponential convergence does not hold). On the other hand, $b > 0$ for finite dimensional kernels, but it could be so close to zero that the global $\bigO(m^{-1/2})$ bound proves tighter.
\citet{Bach2012Herding} also point out these issues, suggesting that another approach is required to fully justify the improved empirical performance of WKH. 

Here, we attempt to address this deficiency by providing an analysis for near-exponential convergence of unconstrained WKH and SBQ. As noted by~\citet{Huszar2012OptimallyWeightedHI}, the weights in WKH do not lie on the unit simplex in many applications, and so the correspondence to Frank--Wolfe does not hold. 
Instead, 
we study the convergence behavior of WKH and SBQ with respect to the \emph{best possible algorithm} under MMD for generating $m$ samples from the target measure --- a feat that is almost certainly unachievable in practice. 
Our analysis effectively says that WKH and SBQ are ``good enough,'' in the sense that one only requires to pick a few more atoms using WKH or SBQ than any other possible algorithm to get close to the performance of the latter. This result is encapsulated in Theorem~\ref{THM:APPROX}.
Within this understanding, we find that the relevant assumption for investigating convergence is \emph{realizability}: that the mean embedding in the kernel space can be exactly reproduced by a linear combination of $r$ samples $\{\bx_i\}_{i=1}^r$. This amounts to assuming that $\pi$ is comprised of finitely many atoms --- for any distribution that can be closely approximated by such a measure, this may well be reasonable to consider. In this case, we show that realizability guarantees $\bigO(e^{-r m})$ convergence.
\subsection{Contributions}
Our contributions in this work include: 
\begin{itemize}
	\item An analysis of two algorithms for approximating expectations --- WKH and SBQ --- highlighting \emph{near exponential} decay with respect to the best possible sampling algorithm under MMD. Our results provide theoretical justifications for empirical observations already made in the literature. 
	
	\item The introduction of the assumption of realizability in the context of KH algorithms, enabling tighter analysis for distributions close to a finitely atomic measure. 
	

	\item A \emph{distributed} algorithm for approximating expectations for large scale computations, together with a short analysis of its convergence properties. To the best of our knowledge, herding has not yet been applied in large-scale distributed settings.
	
	\item Finally, we present \emph{empirical studies} to validate rapid convergence. Since there is ample empirical evidence supporting the good performance of the KH/SBQ algorithms in earlier works~\citep{Huszar2012OptimallyWeightedHI,Chen2010SuperSamples,Bach2012Herding}, we focus instead on demonstrating the empirical performance of the distributed algorithm.
\end{itemize}

\subsection{Notation}
We represent vectors as small letter bolds, e.g., $\bu$.
Matrices are represented by capital bolds, e.g., $\bX, \bT$. 
Sets are
represented by sans serif fonts, e.g., $\sfS$; and the complement of a set $\sfS$ is $\sfS^c$. 
A dot product in a reproducing kernel Hilbert space with kernel $k(\cdot,\cdot)$ is represented as $\langle \cdot, \cdot \rangle_k$, and the corresponding norm is $\|\cdot\|_k$. The dual norm is written as $\| \cdot \|_{k*}$. We denote $\{1,2,\ldots, d\} $ by~$[d]$. 

\section{Background} \label{sec:background} 

In this section, we discuss some relevant background for the algorithms and methods at hand. 

\paragraph{Maximum Mean Discrepancy (MMD).} MMD measures the worst-case error between two probability measures $\pi$ and $\nu$ over the unit ball of a reproducing kernel Hilbert space (RKHS) $\mathcal{H}$ with kernel $k$: 
\vspace{-0.2cm}
\begin{align}
&\text{MMD}(\pi,\nu) := \sup_{\substack{f \in \mathcal{H}\\ \|f\| \leq 1}} \left| \int f(x)d\pi(x) - \int f(x) d\nu(x) \right| \nonumber \\ &= \bigg| \iint k(x,y) d\pi(x)d\pi(y) + \iint  k(x,y) d\nu(x)d\nu(y) \nonumber\\ &\qquad\qquad\qquad -2 \iint k(x,y) d\pi(x) d\nu(y)\bigg| 
\\&= \| \mu_\pi - \mu_\nu \|_k 
\end{align}
where $\mu_\pi$ and $\mu_\nu$ are the respective mean kernel embeddings of $\pi$ and $\nu$. MMD($\pi,\nu$) $\ge 0$; and if $\mathcal{H}$ is universal, then MMD($\pi,\nu$)$=0$ if and only if $\pi\equiv \nu$. We refer to~\citet{Sriperumbudur2010,Gretton2007} for further details. The sampling algorithms we consider approximate $\pi$ by constructing an empirical measure $\nu$ that is close to $\pi$ under~MMD.

\paragraph{Weighted Kernel Herding.}

\begin{algorithm}[H]
	\caption{Weighted Kernel Herding : WKH($\cX$, $m$), or Sequential Bayesian Quadrature : SBQ($\cX$, $m$) }
	\label{algo:wkh}
	\centering
	\begin{algorithmic}[1]
		\STATE  \textbf{INPUT:}  kernel function $k(\cdot,\cdot)$, number of iterations $m$
		\STATE $\sfS = \emptyset$. // \emph{Build solution set $\sfS$ greedily.}	
		\FOR{$n=1\ldots m$}
		\STATE Use~\eqref{eq:LMO} for WKH, or Use~\eqref{eq:LMO_SBQ} for SBQ, to get $\bx_n$. $\sfS = \sfS \cup \{\bx_n\}$
		
		\STATE Update weights $\bw = \bK^{-1} \bz$, 
		where $\bK_{rs} = k(\bx_r, \bx_s)$ for $r,s \in [1,n]$
		\ENDFOR
		\STATE return $\sfS$, $\bw$
	\end{algorithmic}
\end{algorithm}

Recall that our goal is to approximate the expectation of a function $f$ over some probability measure $\pi$ using a weighted empirical measure:
\begin{equation}
\bbE_\pi f(\bx) = \int_\cX f(\bx)d\pi(\bx) \approx \sum_{i=1}^n w_i f(\bx_i),
\label{eq:approxexpectation}
\end{equation}
where $w_i$ are the weights associated with function evaluations at $\bx_i$. For example, taking weights $w_i = \nicefrac{1}{n}$ and samples $\bx_i$ to be independent recovers crude MC integration. Both Kernel Herding \citep{Chen2010SuperSamples} and Quasi Monte Carlo~\citep{Dick2010book}  use $w_i=\nicefrac{1}{n}$ with dependent samples $\bx_i$. 
For brevity, in the sequel, we let $\sfS_j:=\{\bx_1, \bx_2,\ldots, \bx_j\}$ be the collection of the first $j$ samples, and $z(\sfS_j):= \sum_j w_j f(\bx_j)$. 

Here, we present a brief overview of WKH and SBQ, and point the reader to~\citet{Huszar2012OptimallyWeightedHI,Chen2010SuperSamples} for further details. 
As mentioned, these algorithms construct a weighted empirical measure by minimizing MMD to the target measure $\pi$. Greedy algorithms for constructing approximations of the form (\ref{eq:approxexpectation}) --- including WKH and SBQ --- typically involve two alternating steps: (1) generate a sample $\bx_i$; and then (2) compute the weights $w_j$ across all drawn samples $\{\bx_j\}_{j\leq i}$. 
KH chooses the next sample by minimizing MMD, taking all weights to be uniform, that is, $w_i=\nicefrac{1}{n}$ for $i=1,\dots,n$. 
More precisely, given $n$ samples $\{\bx_i\}_{i=1}^n$, the next sample $\bx_{n+1}^{KH}$ is generated according to
\begin{equation}\label{eq:LMO} 
\bx^{KH}_{n+1} = \argmin_{\bx \in \cX} \frac{n}{n+1} \sum_{i=1}^n w_i k(\bx, \bx_i) - 2\bbE_{\bx'\sim \pi}{[k(\bx,\bx')]}. 
\end{equation} 
The SBQ algorithm is slightly more involved.
Consider imposing a functional Gaussian Process (GP) on $f$ with covariance kernel $k$, that is, $f \sim GP(0,k)$. Doing so, the quantities in~\eqref{eq:approxexpectation} become random variables. The sample $\bx_i$ is then chosen to minimize posterior variance,  while the corresponding weights are calculated from the resulting posterior mean.
More precisely, suppose $\bx_1,\dots,\bx_n$ are previously drawn samples. 
A standard result in the theory of GPs asserts that the posterior of $f$, conditioned on the evaluations $\{f(\bx_i)\}_{i=1}^n$, has expectation
\[
\hat{f}(\bx) = \bk^\top \bK^{-1} \bff,
\]
where $\bff$ is the vector of function evaluations $(f(\bx_i))_{i=1}^n$, $\bk$ is the vector of kernel evaluations $(k(\bx, \bx_i))_{i=1}^n$, and $\bK := (k(\bx_i,\bx_j))_{i,j=1}^n$ is the kernel Gram matrix. Consequently, inserting $\hat{f}$ into~\eqref{eq:approxexpectation}, it becomes clear that $\bbE z(\sfS_n) = \bz^\top \bK^{-1} \bff $, where $\bz_i  := \int k(\bx, \bx_i) d\pi(\bx)$. In particular, observe that the weights in~\eqref{eq:approxexpectation} are given by $w_i = \sum_j \bz_j [\bK^{-1}]_{ij}$. 
The posterior variance becomes
\[
\text{cov}(\bx,\by) = k(\bx,\by) - \bk^\top \bK^{-1}\bk.
\]
%
Therefore, we can write the variance of $z(\sfS_n)$ as
\begin{equation}
\label{eq:sbq_costfunction}
\text{var}(z(\sfS_n))  = \iint  k(\bx,\by) d\pi(\bx) d\pi(\by) - \bz^\top \bK^{-1} \bz. 
\end{equation}

Given $n$ samples $\{\bx_i\}_{i=1}^n$, the SBQ algorithm generates the next sample $\bx_{n+1}^{SBQ}$ by minimizing the posterior variance of the approximated integral $z(\sfS_n)$:
\begin{equation}
\label{eq:LMO_SBQ}
\bx^{SBQ}_{n+1} = \argmin_{\bx \in \cX}\text{var}(z(\sfS_{n} \cup \{\bx\})).
\end{equation}
It has been shown that SBQ is equivalent to minimizing MMD with respect to both samples and weights~\citep{Huszar2012OptimallyWeightedHI}.

In this work, we first analyze WKH (Algorithm~\ref{algo:wkh}) which performs the update (\ref{eq:LMO}) for the samples while also updating weights using SBQ's  $w_i = \sum_j \bz_j [\bK^{-1}]_{ij}$. 
The same objective $\bx_{n+1}^{KH}$ was also considered by~\citet{Bach2012Herding}, with the additional constraint that the weights $w_i$ are positive and sum to unity. They use the connections to the classic  algorithm of~\citet{Frank1956bt} to establish convergence rates under this additional constraint. 

\section{Related Works} 
The connection between WKH and the Frank--Wolfe algorithm was further studied in~\citet{Briol2015FWBayesianQuadrature}, providing new variants and rates. Other variants of the Frank--Wolfe algorithm~\citep{jaggi13FW,LacosteJulien:2015wj} enjoy faster convergence at  the cost of additional memory requirements. Specifically, instead of just selecting sample points, one can think of removing bad points from the set already selected. This variant of Frank--Wolfe, known as \emph{FW with away steps}, is one of the more commonly used in practice, because it is known to converge faster. If the weights are not restricted to lie on the simplex, the analogy to matching pursuit algorithms is obvious --- we refer to~\citet{Locatello:2017tq} for corresponding convergence rates. However, the linear rates for these algorithms require bounding certain geometric properties of the constraint set, and this may not be straightforward for RKHS-based applications that usually employ KH and/or SBQ. 
There have been some recent works that provide sufficient conditions for fast convergence of SBQ under additional assumptions on sufficient exploration of point sets~\citep{Motonobu2019}, or special cases of kernels (e.g. the ones generating Sobolev spaces~\citep{santin2016convergence}).  Our setting is more general. Other studies discuss dimension-dependent convergence rates~\citep{Briol_2019} for infinitely smooth functions. Such rates often take the form $\bigO(n^{-1/d})$ for dimension $d$ --- our results have no explicit dimension-dependence. More recently,~\citet{kanagawa_convergence_2020} study convergence properties in misspecified settings. Applying our ideas to these settings may provide an interesting direction for future work.

With the goal of interpreting blackbox models,~\citet{Khanna2019} recently exploited connections with submodular optimization to provide the weaker forms of convergence we discuss (in Section~\ref{sec:realizability}), in the form of an approximation guarantee for SBQ for discrete $\pi$. Our result is more general, allowing for arbitrary probability measures. In fact, we do not make use of submodular optimization results. A similar proof technique was also used by~\citet{Khanna2017icml} for proving approximation guarantees of low rank optimization. The proof idea for the distributed algorithm was inspired from tracking the optimum set, which is a common theme in analysis of distributed algorithms in discrete  optimization (see, e.g.,~\citet{AltschulerBFMRZ16} and references therein). 

\section{Convergence Results} \label{sec:theory}

In this section, we present our convergence results. The starting point of our analysis is the following re-interpretation of the posterior variance minimization~\eqref{eq:sbq_costfunction} as a variational optimization of MMD~\citep{Huszar2012OptimallyWeightedHI}. We can re-write~\eqref{eq:sbq_costfunction} as a function of a set of chosen sample indices $\sfS$ and weights $\bw$:
\begin{eqnarray}
\label{eq:MMDcostfunction}
g(\sfS, \bw) := \| \mu_\pi - \sum_{i \in \sfS} w_i \phi(\bx_i) \|_k^2,
\end{eqnarray}
where $\phi$ is the respective feature mapping, i.e., $k(\bx,\by) = \langle \phi(\bx), \phi (\by) \rangle$, and where $\mu_\pi : = \int_{\by \sim \pi} \phi(\bx) d \pi(\bx) $ is the mean embedding in the kernel space. 

Before presenting our results, we delineate the assumptions we make on the cost function~\eqref{eq:MMDcostfunction}. There is an implicit assumption that $f$ lies within the chosen RKHS corresponding to the kernel $k$. Otherwise, there is no guarantee that any algorithm can approximate $f$ within $\epsilon$ error, and the discussion of convergence rates is meaningless. 


\begin{assumption}[Convexity/Smoothness]\label{assump:rsc} We assume that the loss function $g(\cdot)$ is $m_\omega$-restricted strongly convex and $M_\Omega$-smooth over $\sfS^\star_\perp \cup \sfS_n$, with $m_\omega > 0$ and $M_\Omega < \infty$. In other words, for $\sfS_1, \sfS_2 \subset \sfS_r^\star \cup \sfS_n$, with $Z(\sfS_j) = \sum_j w_j \phi(\bx_j)$, $Z_{12} = Z(\sfS_1) - Z(\sfS_2)$, and $D(\sfS_1, \sfS_2) =g(\sfS_1) - g(\sfS_2)- \langle \nabla g(\sfS_2), Z(\sfS_1) - Z(\sfS_2) \rangle_k $,

\begin{eqnarray*}
     \frac{m_\omega}{2} \| Z_{12}\|_k^2 \leq  D(\sfS_1, \sfS_2)   \leq   \frac{M_\Omega}{2} \| Z_{12} \|_k^2 .
\end{eqnarray*}

\end{assumption}

%


Assumption~\ref{assump:rsc} is standard for characterizing geometric rates in optimization; e.g., see~\citet[Section 9.3.1]{wainwright2019high}, and is often implicitly assumed, e.g., by~\citet{Bach2012Herding}, when drawing upon the Frank--Wolfe connection. This assumption or a slight variation of it, is not only sufficient but also necessary to bound geometric convergence rates. For our case, one can equivalently view the assumption as enforcing that the kernel matrix of the atoms $\sfS^\star_\perp \cup \sfS_n$ has a minimum eigenvalue bounded away from zero for any fixed~$n$. 

\begin{assumption}[Standardization] \label{assump:standardization} We assume that the feature mapping is standardized, that is, $k(\bx,\bx) = 1$ for all $\bx \in \cX$.

\end{assumption}

Assumption~\ref{assump:standardization} is not restrictive. Rather, as in previous works~\citep{Huszar2012OptimallyWeightedHI}, it is enforced to avoid otherwise unnecessary terms for ease of exposition.  Any kernel can be standardized over its support, and remain a kernel. 
Assumptions~\ref{assump:rsc} and~\ref{assump:standardization} are satisfied by many commonly used kernels. Any finite dimensional normalized kernel satisfy them, including the cosine kernel or polynomial kernels; and general normalized RBF kernels on compact supports. 
We are now ready to present our main result regarding the near linear convergence of the discrepancy metric $g$.

\begin{theorem}[Approximation Guarantee]\label{THM:APPROX} Suppose that Assumptions~\ref{assump:rsc} and~\ref{assump:standardization} are satisfied. Let $\sfT_r := \argmin_{|\sfS| \leq r} \min_\bw g(\sfS,\bw)$. For $0<\epsilon<1$, consider $s=(r \frac{M_\Omega}{m_\omega}\log \frac{1}{\epsilon})$ iterations of Algorithms~\ref{algo:wkh} (WKH or SBQ), returning the set $\sfS_s$. Then, $ \min_\bw g(\sfS_s, \bw) ) \leq (1- \epsilon) [ \min_\bw g(\sfT_r, \bw)] + \epsilon \| \mu_\pi\|_k^2 $. 
\end{theorem}
 
The proof for WKH is presented in Appendix~\ref{sec:appendixApproxGuaranteeProof}. Theorem~\ref{THM:APPROX} is general, and it is applicable to kernels of any dimension, including infinite dimensional kernels. Intuitively, it may be hard to claim that an infinite dimensional embedding can be closely approximated by a finite number of atoms without additional strong assumptions. However, our analysis gets around this complexity by comparing the performance of the algorithm at hand with best possible finite set of $r$ atoms. Indeed, Theorem~\ref{THM:APPROX} is somewhat weaker than a classical convergence result.  Instead of providing a rate on closeness to the optimum after $k$ iterations, it provides a contraction factor on how close the algorithm gets to the best possible $r$ steps that \emph{any} algorithm could have taken. The theoretical best possible $r$ samples could even come from an exhaustive combinatorial and computationally hard search over the entire space. However, by choosing only a multiplicative $\bigO(\log \frac{1}{\epsilon})$ number of extra atoms through the greedy selection processes in WKH or SBQ, we can get provably close to best-case performance. 


The SBQ algorithm is equivalent to Algorithm~\ref{algo:wkh}, after replacing~\eqref{eq:LMO} in Step 4 with~\eqref{eq:LMO_SBQ}. 
Thus, WKH solves a linear program every iteration, while SBQ solves a kernel $\ell_2$ minimization problem. The decrease in the cost function~\eqref{eq:MMDcostfunction} per iteration of SBQ is more than its decrease per iteration of WKH. Thus, Theorem~\ref{THM:APPROX} also recovers the special case of SBQ for discrete densities studied by~\citet{Khanna2019} by exploiting connections to weak submodularity. Their result is also an approximation guarantee (not global convergence guarantee), and only applies for SBQ. Our result is also valid for WKH and provides an alternative proof for SBQ without exploiting weak submodularity. 
%
%

\subsection{Realizability}
\label{sec:realizability}

Theorem~\ref{THM:APPROX} in itself is quite general -- it holds for any kernel of arbitrary dimensionality. Here, we further specialize Theorem~\ref{THM:APPROX} and provide a sufficient condition under which the convergence rate is geometric (instead of being near-geometric). We encapsulate this sufficient condition as the assumption of realizability. 

\begin{assumption}[Realizability]
\label{assump:realizability}
There exists a set $\sfS^\star_r$ of $r$ atoms, and weights $\bw^\ast$, such that $ g(\sfS^\star_r, \bw^\ast) = 0$.
\end{assumption}

Assumption~\ref{assump:realizability} posits that there exists a set of atoms in the mapped domain $\phi(\cX)$ whose weighted average exactly evaluates to the expectation $\mu_\pi$ so that the discrepancy $g$ is $0$. If the RKHS is universal, this is equivalent to assuming that $\pi$ is finitely atomic. Otherwise, realizability can be achieved with finite-dimensional kernels. By considering the assumption of realizability, we are able to investigate convergence rates independently from the capacity of the target measure $\pi$ to be approximated (under MMD) by an atomic distribution. Furthermore, in some common use-cases, such as the data summarization task in Section \ref{sec:Experiments}, the target measure $\pi$ is finitely atomic, in which case, realizability is automatically satisfied.

\begin{theorem}
	\label{THM:WKHLINEAR} Under Assumptions 1 through 3, if $\sfS_i$ is the sequence of iterates produced by Algorithm~\ref{algo:wkh} (WKH or SBQ), the function $g$ converges as $\min_w g(\sfS_i, \bw) \leq \exp(- \frac{i m_\omega}{r M_\Omega}) g(\emptyset, 0)$.
\end{theorem} 

\noindent
\emph{Proof Outline.} The central idea of the proof is to track and bound the selection of a sample at each iteration, compared to the ideal selection $\sfS^\star_r$ that could have provided the optimum solution. For this purpose, the properties of the selection subproblem~\eqref{eq:LMO} and the assumptions are  used. The detailed proof is presented in the Appendix~\ref{sec:appendixWKHLinearProof}.

Theorem~\ref{THM:WKHLINEAR} provides a linear convergence rate for \wkh~ under the conditions specified in Assumptions 1 through 3. Recall that~\citet{Bach2012Herding} also provided a linear rate for finite dimensional kernels by drawing on the equivalence of the herding algorithm to the Frank--Wolfe algorithm. Their rate is $\bigO(\exp{(-b^2m/R^2)})$, where $b$ is the distance of the optimum from the boundary of the marginal polytope and $R$ is the width of the marginal polytope. Our result is independent of these constants. As long as Assumption~\ref{assump:realizability} is satisfied, our work shows that exponential convergence is guaranteed. To the best of our knowledge, such a sufficient condition for these harder cases has not been previously been established.

Let us also briefly discuss the case where Assumption \ref{assump:realizability} is not satisfied, but $\pi$ is $\epsilon$-close under MMD to a finitely atomic measure $\tilde{\pi}$ that \emph{does} satisfy Assumption \ref{assump:realizability}. Using the triangle inequality for MMD, it is straightforward to show that $\min_w g(\sfS_i, \bw) \leq \exp(- \frac{i m_\omega}{r M_\Omega}) g(\emptyset, 0) + 2 \epsilon$, suggesting near-exponential convergence in these settings. Indeed, in cases where $\epsilon$ is small relative to $g(\emptyset, 0)$, this could explain the excellent empirical performance seen for WKH and SBQ.

\section{Distributed Kernel Herding}
\label{sec:Distributed}
In many cases, the search over the domain $\cX$ can be a severe computational bottleneck to practical use. In this section, we develop a new herding algorithm that can be distributed over multiple machines and run in a streaming map-reduce fashion. We also provide a quick convergence analysis, using techniques presented in Section~\ref{sec:theory}.

The algorithm proceeds as follows. The domain $\cX$ is split onto $s$ machines uniformly at random. Each of the $s$ machines has access to only $\cX_i \subset \cX$, such that $\bigcup_i \cX_i = \cX$ and $\cX_i \cap \cX_j = \emptyset $ for $i \neq j$. Each machine runs its own herding algorithm (Algorithm~\ref{algo:wkh}) independently, by restricting the search space in~\eqref{eq:LMO} to $\cX_i$, instead of $\cX$. Finally, the iterates generated by each machine are sent to a central server machine. The server \emph{collates} the samples by running another copy of the same algorithm, with $\cX$ replaced by the discrete set of samples received. Finally, the best solution out of all the $s+1$ solutions is returned. The pseudo-code is illustrated in Algorithm~\ref{algo:distributed}. In what follows, we provide a convergence guarantee in the case of realizability for a single atom.

\begin{algorithm}[H]
	\caption{Distributed Kernel Herding: Dist($\cX$, $k$)}
	\label{algo:distributed}
	\centering
	\begin{algorithmic}[1]
		\STATE  \textbf{INPUT:}kernel function $k(\cdot,\cdot)$, number of iterations $k$
		\STATE Partition $\cX \supset \cX_i, i \in [s]$ uniformly at random and transmit $\cX_i$ to machine $i$
		\STATE // \emph{Receive solution sets}
		\STATE $\sfG_i, \bw_i$ $\leftarrow$ WKH($\cX_i$, $k$)  // \emph{run in parallel $\forall i \in [s]$}
		\STATE $\cY = \cup_i \sfG_i$
		\STATE $\sfG_{s+1}, \bw_{s+1} \leftarrow$ WKH($\cY$, $k$)
		\STATE $i^\star \leftarrow \argmin_{i \in [s+1]} g(\sfG_i,\bw_i)$
		\STATE return $\sfG_{i^\star}$, $\bw_{i^\star}$ 
		\end{algorithmic}
\end{algorithm}

\begin{theorem}
	\label{THM:DISTRIBUTED}  Under Assumption 1 with $r=1$, and Assumptions 2 and 3, if $\sfS_i$ is the sequence of iterates produced by Algorithm~\ref{algo:distributed}, the function $g(\cdot)$ converges as $\bbE \min_\bw g(\sfS_i ,\bw) \leq \exp(- \frac{i m_\omega}{r M_\Omega}) g(\emptyset, 0)$.  
\end{theorem}

\noindent
\emph{Proof Sketch:} Note that the final set of filtered iterates outputted are the \emph{best} out of $(s+1)$ possible sequences.  The proof tracks the possibilities for $\sfS_1^\star$ when $\cX$ is split. The goal is to then show that under all possible scenarios, at least one of the sequences converges linearly. The convergence of individual sequences is based on the proof techniques used in proof of Theorem~\ref{THM:WKHLINEAR}. 


We remark that, for the more general case of $r>1$, our proof technique does not give a non-trivial convergence rate. 
This is likely an artifact of our proof technique, and improving it is an interesting open question for future research. Nevertheless, as we shall see in the following section, the algorithm displays improved performance in practice.

\section{Experiments}\label{sec:Experiments}



We refer the readers to earlier works for empirical evidence of the relative performance of the Herding and SBQ algorithms~\citep{Bach2012Herding,Chen2010SuperSamples,Welling2010,Welling2009a,Welling2009b,Huszar2012OptimallyWeightedHI,Ghahramani2002BayesianMC}.
In this section, we focus on studying the distributed versions of these algorithms to illustrate the speed / accuracy tradeoff.

\subsection{Matching a Distribution}
In this study, our goal is to show the tradeoff between performance and scalability when WKH/SBQ are distributed over multiple machines. Towards this end, we extend an experiment considered in~\citet{Chen2010SuperSamples,Huszar2012OptimallyWeightedHI}. In this experiment, the target density is a mixture of $20$ two-dimensional Gaussian distributions. Samples are chosen by the contending algorithms, and the MMD distance of the sample distribution to the target distribution is reported for different number of samples. 

\begin{figure}[h]

	\centering
	\includegraphics[scale=0.5]{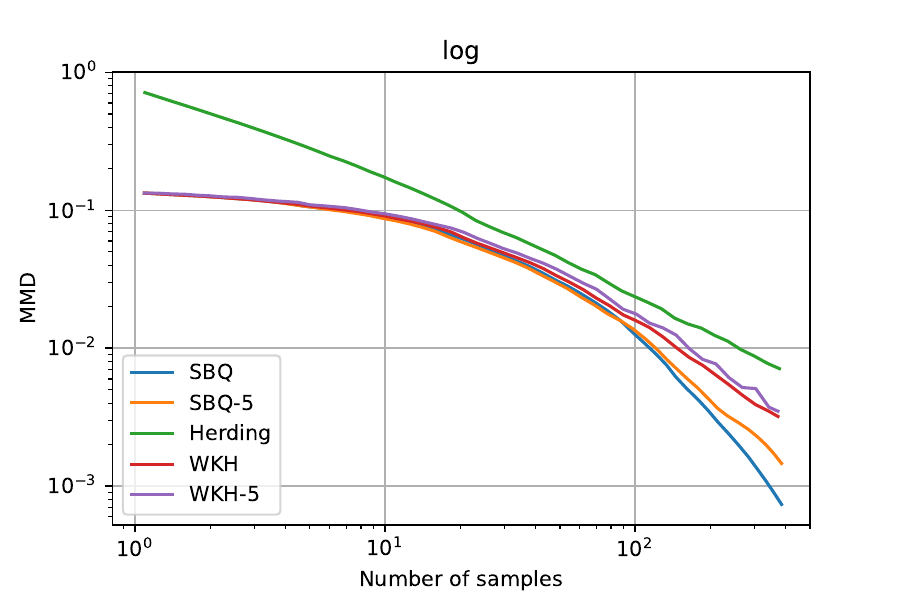}
	\caption{Simulated data results for distributed SBQ/WKH. Herding is WKH with uniform weights. SBQ/WKH are single machine algorithms, while SBQ-5/WKH-5 are their respective distributed versions.}
	\label{fig:simulated}

\end{figure}

The sampling subroutine (step 4 in Algorithm~\ref{algo:wkh}) requires solving an optimization problem over a continuous domain. To make the problem easier,~\citet{Chen2010SuperSamples} and \citet{Huszar2012OptimallyWeightedHI} first select 10000 points uniformly at random as the set of potential candidates. They note that this does not affect the performance of chosen samples by much. We also adopt the same methodology with the additional step of arbitrarily partitioning these points over $s$ machines for $s=1,5$. Our objective is to illustrate the degradation of performance due to partitioning in Algorithm~\ref{algo:distributed}.

The results are reported in Figure~\ref{fig:simulated}. SBQ-5 and WKH-5 are distributed versions of SBQ and WKH respectively, run over $5$ machines. The labels SBQ and WKH are for their respective single machine versions. Since the search space is split and the search step is parallelized, we receive a five-fold speedup in the algorithms, with a graceful degradation in the reported MMD. We notice a similar pattern of degradation for larger number of machines, but this is omitted from the graph to avoid overcrowding.

\subsection{Data Summarization}
In this study, we apply Algorithm~\ref{algo:distributed} to the task of data summarization under logistic regression, as considered by~\citet{HugginsCB16}. The task of data summarization is as follows. The goal is to select a few data samples that represent the data distribution \emph{sufficiently} well, so that a model built on the selected subsample of the training data does not degrade too much in performance on unseen test data. More specifically, we are interested in approximating the test distribution (i.e., discrete $\pi$) using a few samples from the training set. Hence, algorithms such as SBQ and WKH are applicable, provided we have a reasonable kernel function. Recently,~\citet{Khanna2019} used SBQ with Fisher Kernels~\citep{Jaakkola1999} for this task. By using the distributed SBQ/WKH over $s$ machines,  we obtain a roughly $s$-times speedup on the run time with minimal loss in the log-likelihood of the selected sample when compared to the results of~\citet{Khanna2019} for this task across three different datasets namely \texttt{ChemReact}, \texttt{CovType} and \texttt{WebSpam}. 
\begin{figure}
\centering
	\begin{subfigure}[b]{0.45\textwidth}
		\centering
		\includegraphics[scale=0.34]{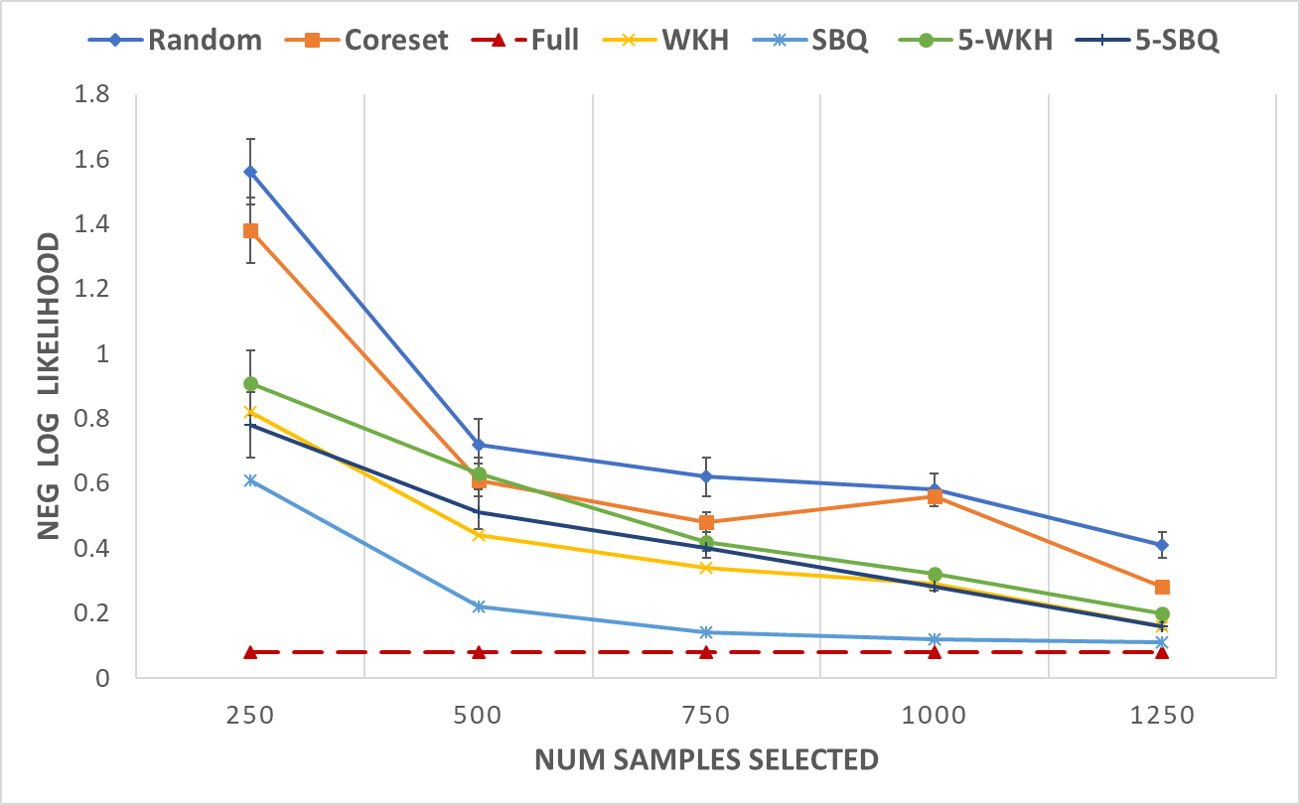}
		\caption{\texttt{ChemReact}}
	\end{subfigure}
		~
	\begin{subfigure}[b]{0.45\textwidth}
		\centering
			\includegraphics[scale=0.34]{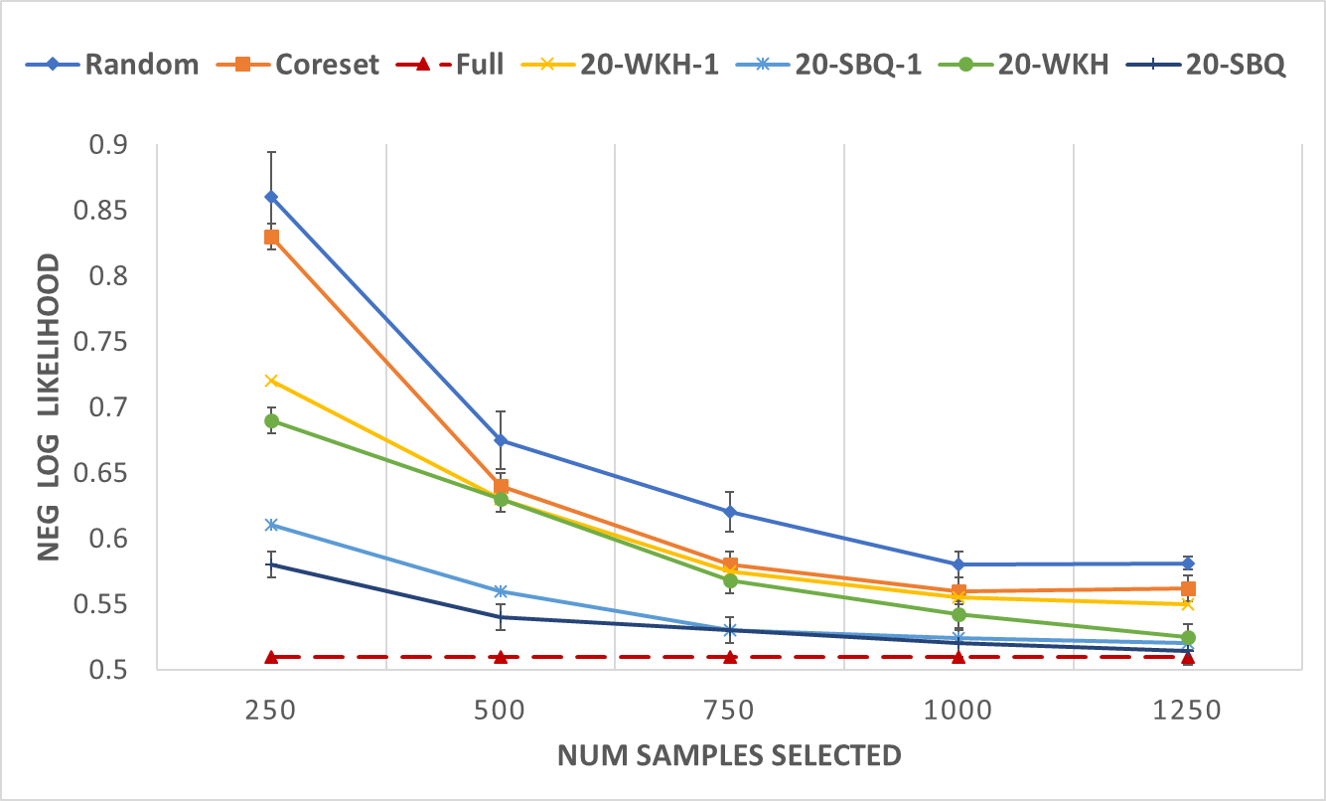}
		\caption{\texttt{CovType}}
	\end{subfigure}	

	\vspace{4mm}
	\begin{subfigure}[b]{0.45\textwidth}
		\centering
		\includegraphics[scale=0.34]{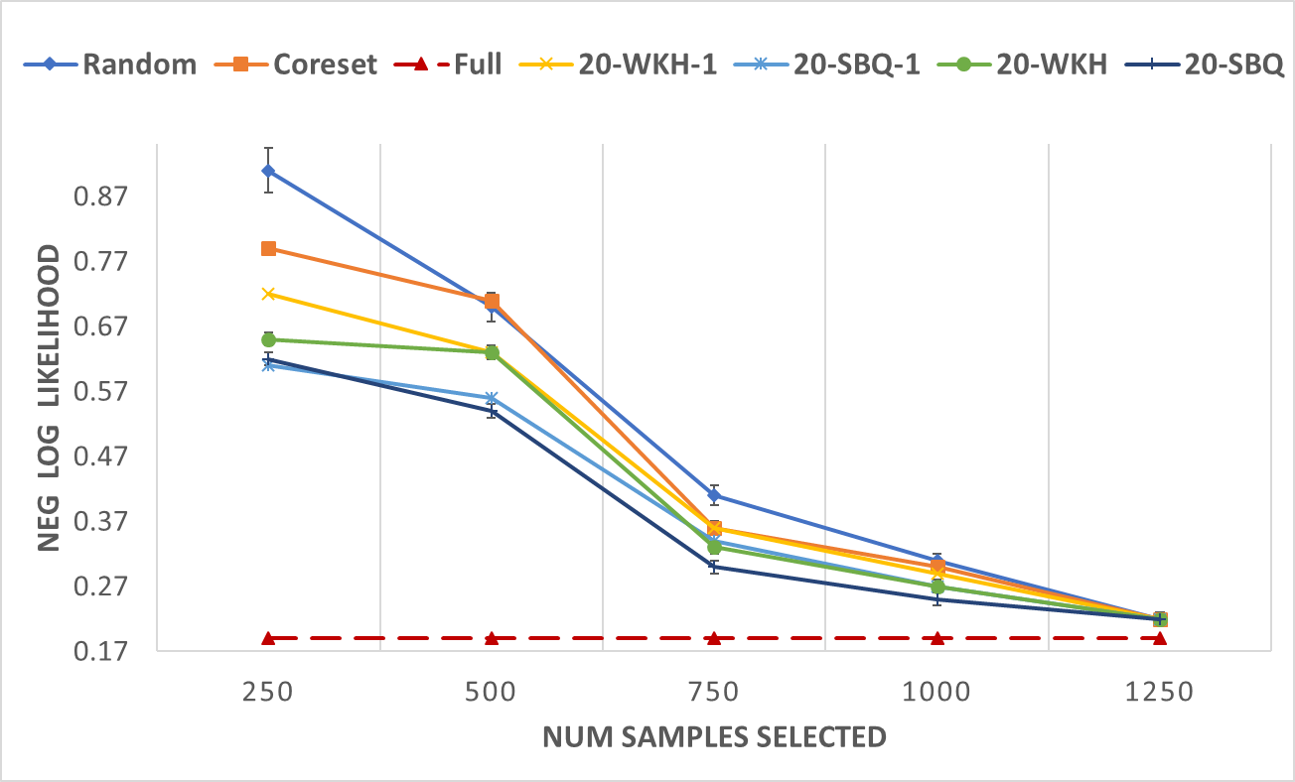}
		\caption{\texttt{WebSpam}}
	\end{subfigure}
	\caption{Performance for logistic regression over three datasets for variants of sampling methods. `Full' reports the numbers for training with the entire training set. For the algorithms $s$-SBQ and $s$-WKH, $s$ represents the number of machines used to select the samples. WKH and SBQ represent single machine versions of the algorithms for the smallest dataset, the other two datasets were too big to run on a single machine. The s-WKH-1 experiment is obtained by using only the output of a single split out of $s$ of the dataset run on only one of the machines. Across the three different datasets, the distributed versions of SBQ/WKH proposed in this paper show minimal loss in accuracy while achieving almost linear speedup.}
	\label{fig:coresets}
\end{figure}

\paragraph{Fisher Kernels.} For completion, we provide a brief overview of constructing the Fisher kernels. Suppose that we have a parametric model that we learn using maximum likelihood estimation, i.e.,  $\hat{\theta} := \arg\max \log p(\bX | \theta) $, where $\theta$ represents the model parameters and $\bX$ the data. The notion of similarity that Fisher kernels employ is that 
if two objects are \emph{structurally} similar as the model \emph{sees them}, then slight perturbations in the neighborhood of the fitted parameters $\hat{\theta}$ would impact the fit of the two objects similarly. In other words, the feature embedding $\bff_i := \frac{\partial \log p(\bX_i | \theta)}{\partial \theta} |_{\theta = \hat{\theta}}$, for an object $\bX_i \rightarrow \bff_i$ can be interpreted as a \emph{feature mapping} which can then be used to define a similarity kernel by a weighted dot product:
\[ 
\kappa(\bX_i, \bX_j) := \bff_i^\top \cI^{-1} \bff_j,
\]
where $\cI:= \bbE_{p(\bX)}[ \frac{\partial  \log p(\bX | \theta)}{\partial \theta}^\top\frac{\partial \log p(\bX | \theta)}{\partial \theta}]$ is the Fisher information matrix. 
The information matrix serves only to re-scale the feature space, and can be taken to be the identity without sacrificing any asymptotic properties~\citep{Jaakkola1999}. 
The corresponding kernel is then called the \emph{practical} Fisher kernel, and is often used in practice.

Another method for training data summarization is that of coreset selection~\citep{HugginsCB16}, which aims to reduce the training data size for optimization speedup while retaining guaranteed approximation to the training likelihood. In line with its objective, coreset selection algorithms are required to be fast.
Therefore, for efficiency, the coreset selection algorithm is usually closely tied with the respective model, as opposed to being a model-agnostic. 

We employ different variants of WKH/SBQ to the problem of training data summarization under logistic regression, as considered by~\citet{HugginsCB16} using coreset construction. We experiment using three datasets \texttt{ChemReact}, \texttt{CovType} and \texttt{WebSpam}. \texttt{ChemReact} consists of $26,733$ chemicals each of feature size $100$. Out of these,  $2500$ are test data points. The prediction variable is $0/1$ and signifies if a chemical is reactive. \texttt{CovType} has $581,012$ examples each of feature size $54$. Out of these, $29,000$ are test points. The task is to predict whether a type of tree is present in each location or not. \texttt{WebSpam} has 350,000 webpages each having $127$ features. Out of these, $50,000$ are test data points. The task here is to predict whether a webpage is spam or not. We refer to~\citet{HugginsCB16} for source of the datasets.

In each of the datasets, we further randomly split the training data into $10\%$ validation and $90\%$ training. We train the logistic regression model on the new training data, and we use the validation set as a proxy to the unseen test set distribution. We build the kernel matrix $\bK$ and the affinity vector $\bz$, and we run different variants of sampling algorithms to choose samples from the training set to approximate the discrete validation set distribution in the Fisher kernel space. Once the training set samples are extracted, we rebuild the logistic regression model only on the selected samples, and we report negative test likelihood on unseen test data to show how well has the respective algorithm built a model specific dataset summary.

\texttt{ChemReact} is small enough to fit on a single machine, so we run WKH and SBQ on a single machine. To present the tradeoff, we also run $5$-WKH and $5$-SBQ.
These are about five times faster than their single machine counterparts, but they degrade in predictive performance. 
Our aim is to compare distributed WKH/SBQ against their single machine counterparts. For completeness, we also include  the results of coreset selection algorithm and  random data selection as implemented by~\citet{HugginsCB16}, since these algorithms were tested by~\citet{Khanna2019} on the same problem. To keep the focus on our goal of distributing WKH/SBQ, we do not compare against other coreset algorithms, since coreset construction is not the central goal of this paper. We note that generally SBQ has better performance numbers than WKH for same $k$ across different values of $k$. Note that \texttt{WebSpam} and \texttt{CovType} were too big to run on a single machine, and they are thus perfect examples to illustrate the impact and usefulness of the distributed algorithm. All the experiments were run on 12-core 16GB RAM machines. For all the experiments we conducted, the variance over multiple runs of the distributed algorithm was very low (almost $0$), and the trend of relative performance remained the same.  

The results are presented in Figure~\ref{fig:coresets}. The algorithms we run are WKH, SBQ, $s$-SBQ and $s$-WKH, where $s$ represents the number of machines used to select $m$ samples for different values of $m$. The s-WKH-1 experiment uses only the output of a single split out of $s$ draws from the dataset, run on only one of the machines. 
For completeness, we also include ``Random'' which selects the data points uniformly at random, and ``Coreset'' which was proposed by~\citet{HugginsCB16}.

%

%

\section{Conclusion} 

We have analysed two existing algorithms --- WKH and SBQ --- as well as a new distributed algorithm for estimating expectations. 
Our results help to bridge the gap between theory and empirical performance by showing that these algorithms perform comparably to the theoretical best possible sampling method over MMD, and exhibit geometric convergence rates for finitely atomic target measures.
Our realizability assumption is the key insight that allows us to improve upon previous results. 
However, we were unable to develop convergence rates for the distributed algorithm over arbitrary atomic target measures using the techniques presented. Developing new methods to address this is the subject of future work.


\bibliography{bibliography}
\appendix

\newcommand{\rajiv}[1]{\textcolor{red}{Rajiv:\ #1}}

\appendix\section{Appendix}
\label{sec:appendix}




\subsection{Proof of Theorem~\ref{THM:WKHLINEAR}}
\label{sec:appendixWKHLinearProof}
We begin by first proving Theorem~\ref{THM:WKHLINEAR}, since the additional assumption of realizability makes it an easier read. For further ease of exposition, instead of directly working with $g(\cdot)$, we translate the function to remove any constants not dependent on the variable. We write, 
\[ 
l( \sfS) := \|\mu_\pi\|_k^2 - g(\sfS)  = \bz^\top \bK^{-1} \bz.
\]
Some auxiliary Lemmas are proved later in this section. We use $Z(\sfS_j):= \sum_j w_j \phi(\bx_j)$ Further, note that the Assumption 2, when applied for $h(\cdot)$, ensures that for any iterates considered in this proof we have that 
\begin{align*}
& \hspace{-10mm} - \frac{m_\omega}{2} \| Z(\sfS_i) - Z(\sfS_j) \|_k^2 \\& \ge  l(\sfS_i) - l(\sfS_j) - \langle \nabla l(\sfS_i), Z(\sfS_i) - Z(\sfS_j) \rangle_k \\ &\geq  - \frac{M_\Omega}{2} \| Z(\sfS_i) - Z(\sfS_j) \|_k^2 .
\end{align*}

\begin{proof}
	Say $(i-1)$ steps of the Algorithm~\ref{algo:wkh} have been performed to select the set $\sfS$. Let $\bw \in \bbR^{(i-1)}$ be the corresponding weight vector. At the $i^{\text{th}}$ step, $\bx_i$ is sampled as per~\eqref{eq:LMO}. Let $h(\sfS, \bu) := \| \mu_\pi \|_k^2 - \| \mu_\pi - \sum_j u_j \phi(\bx_j) \|_k^2$, so that $l (\sfS) = \min_{\bu} h(\sfS,\bu)$ (as per Lemma~\ref{lem:weightsAreOptimal}) . Set weight vector $\bu \in \bbR^i$ as follows. For $j \in [0,i-1]$, $u_i = w_i$. Set $u_i = \alpha$, where $\alpha $ is an arbitrary scalar.  
	
	From weight optimality proved in Lemma~\ref{lem:weightsAreOptimal},
	\begin{align*}
	l( \sfS \cup \{\bx_i\}   ) - l (\sfS) 
	\geq h ( \sfS \cup \{\bx_i\},  \bu) - l(\sfS)   ,
	\end{align*}
	for an arbitrary $\alpha \in \bbR$. 
	From Assumption 2 (smoothness),
	\begin{align*}
	l( \sfS \cup \{\bx_i\}   ) - l (\sfS)  \geq  \alpha \langle\nabla l(\sfS), \phi(\bx_i) \rangle_k   - \alpha^2 \frac{M_\Omega}{2}.
	\end{align*}
	Let $\gamma_\sfS$ be the optimum value of the solution of~\eqref{eq:LMO}. 
	Since $\bx_i$ is the optimizing atom, 
	\begin{align*}
	l( \sfS \cup \{\bx_i\}   ) - l (\sfS)  \geq  \alpha \gamma_\sfS   - \alpha^2 \frac{M_\Omega}{2}  .
	\end{align*}
	 Let $\sfS^\star_\perp$ be the set obtained by orthogonalizing $\sfS^\star_r$ with respect to $\sfS$ using the Gram-Schmidt procedure. Putting in $\alpha = \frac{\gamma_\sfS}{M_\Omega} $, we get, 
	\begin{align} \label{eq:lowerboundStep1} 
	l( \sfS \cup \{\bx_i\}   ) - l (\sfS) & \geq  \frac{1}{2 M_\Omega} \gamma_\sfS \\\nonumber
	& \geq  \frac{1}{2 r  M_\Omega} \sum_{\bx_j \in \sfS^\star_\perp} \langle \phi(\bx_j), \nabla l(\sfS) \rangle_k^2 \\\label{eq:lowerboundstep3}
	& \geq \frac{ m_\omega}{r  M_\Omega} \left( l( \sfS \cup \sfS^\star_\perp) - l(\sfS) \right)\\\nonumber
	& \geq \frac{m_\omega}{r  M_\Omega} \left( l(\sfS^\star_r) - l(\sfS) \right)\\\nonumber
	& = \frac{m_\omega}{r  M_\Omega} \left( \|\mu_\pi\|_k^2 - l(\sfS) \right)  .
	\end{align}

The second inequality is true because $\gamma_\sfS = \langle\nabla l(\sfS), \bx_i \rangle_k  $ is the value of~\eqref{eq:LMO} in the $i^{\text{th}}$ iteration. The third inequality follows from Lemma~\ref{lem:upperbound}. The fourth inequality is true because of monotonicity of $l(\cdot)$, and the final equality is true because of Assumption 1 (realizability). 

Let $C:= \frac{m_\omega}{r  M_\Omega}$. We have $l( \sfS \cup \{\bx_i\}   ) - l (\sfS) =  g (\sfS) - g( \sfS \cup \{\bx_i\}   ) \geq C g(\sfS) \implies g(\sfS \cup \{\bx_i\}) \leq (1- C) g(\sfS)  $.
The result now follows. 

\end{proof}

\subsection{Proof of Theorem~\ref{THM:APPROX}}
\label{sec:appendixApproxGuaranteeProof}
\begin{proof}
	
We proceed as in the proof of Theorem~\ref{THM:WKHLINEAR}, but by replacing $\sfS_r^\star$ with $\sfT_r$. From~\eqref{eq:lowerboundstep3}, 

\begin{align*}
l( \sfS \cup \{\bx_i\}   ) - l (\sfS) &\geq  \frac{ m_\omega}{r  M_\Omega} \left( l(\sfT_r ) - l(\sfS) \right).\\\nonumber 
\end{align*}

Adding and subtracting $l(\sfT_r)$ on the LHS and rearranging,

 \begin{align*}
 l(\sfT_r) - l ( \sfS \cup \{\bx_i\}   ) & \leq (1-\frac{ m_\omega}{r  M_\Omega} ) (l(\sfT_r) - l ( \sfS   )).
 \end{align*}

Thus after $k$ iterations, 

 \begin{align*}
l(\sfT_r) - l ( \sfS_k  ) & \leq (1-\frac{ m_\omega}{r  M_\Omega} )^k \left(l(\sfT_r) - l ( \emptyset   )\right).
\end{align*}

Rearranging, 
\begin{align*}
l(\sfS_k) &\geq \left( 1 - (1-\frac{ m_\omega}{r  M_\Omega} )^k \right) l(\sfT_r)\\
&\geq \left(1 - \exp(-\frac{k m_\omega}{r  M_\Omega} )\right) l(\sfT_r).
\end{align*}

With $k=(r \frac{M_\Omega}{m_\omega}\log \frac{1}{\epsilon})$, we get,
\begin{align*}
l(\sfS_k) \geq (1- \epsilon) l(\sfT_r).
\end{align*}
The result now follows.
\end{proof}
\subsection{Auxiliary Lemmas}

The following Lemma proves that the weights $w_i$ in $g(\sfS)$ obtained using the posterior inference are an optimum choice that minimize the distance to $\mu_\pi$ in the RKHS over any set of weights~\citep{Khanna2019}. 

\begin{lemma} 
	\label{lem:weightsAreOptimal} 
The residual $\mu_\pi - \sum_j w_j \phi(\bx_j)$ is orthogonal to $ \bx_i \in \sfS\, \forall i$. In other words, for any set of samples $\sfS$,	$g(\sfS) = \min_{\bu} \| \mu_\pi - \sum_i u_i \phi(\bx_i)\|_k $. 
\end{lemma}
\begin{proof}
	Recall that $w_i = \sum_j [\bK^{-1}]_{ij} \bz_j $, and $\bz_i = \int k(\bx,\bx_i) d\pi(\bx)$. For an arbitrary index $i$, 
	
	\begin{align*}
	& \hspace{-10mm} \langle \mu_\pi - \sum_j w_j \phi(\bx_j), \phi(\bx_i) \rangle_k \\ 
	= & \int k(\bx, \bx_i) d\pi(\bx) - \langle \sum_j w_j \phi(\bx_j), \phi(\bx_i)\rangle_k \\ 
	= & \bz_i  - \langle \sum_j w_j \phi(\bx_j), \phi(\bx_i)\rangle_k \\ 
	= & \bz_i  - \sum_j w_j k(\bx_j, \bx_i)\\ 
	= & \bz_i  - \sum_j \sum_t  [\bK^{-1}]_{tj} \bz_t  k(\bx_j, \bx_i)\\ 
	= & \bz_i  - \sum_t \bz_t \sum_j  \bK_{ji} [\bK^{-1}]_{tj}   \\ 
	= & \bz_i  - \bz_i ,  \\ 
	\end{align*}  
	where the last equality follows by noting that  $\sum_j \bK_{ji}  [\bK^{-1}]_{tj} $ is inner product of row $i$ of $\bK$ and row $t$ of $\bK^{-1}$, which is $1$ if $t=i$ and $0$ otherwise. This completes the proof. 
\end{proof}

\begin{lemma} For any set of chosen samples $\sfS_1$, $\sfS_2$, let $\cP$ be the operator of projection onto span($\sfS_1 \cup \sfS_2$). Then,
	\label{lem:upperbound}
	 $l(\sfS_1 \cup \sfS_2 ) - l(\sfS_1 )\leq \frac{\cP(\nabla l(\sfS_1 ))}{2m_\omega} $. 
\end{lemma}

\begin{proof}
Observe that
	\begin{align*}
	0 &\leq l(\sfS_1 \cup \sfS_2 ) - l(\sfS_1 )\\ & \leq \langle \nabla l(\sfS_1), Z(\sfS_1 \cup \sfS_2) -Z(\sfS_1) \rangle_k \\ & \hspace{5mm} \,\,- \frac{m_\omega}{2}
	\|Z(\sfS_1 \cup \sfS_2)  - Z(\sfS_1) \|_k^2 \\
	& \leq \argmax_{X \in \text{span}(\sfS_1 \cup \sfS_2)} \langle\nabla
		l(\sfS_1 ),X - Z(\sfS_1) \rangle_k
	- \frac{m_\omega}{2} \| X - Z(\sfS_1) \|_k^2\\
	& = \argmax_{X} \langle
	\cP(\nabla l(\sfS_1 )),X - Z(\sfS_1) \rangle_k
	- \frac{m_\omega}{2} \| X - Z(\sfS_1) \|_k^2.	
	\end{align*}
   Solving the argmax problem on the RHS for $X$, we get the required result.

\end{proof}

\subsection{Proof of Theorem~\ref{THM:DISTRIBUTED}}
\label{sec:appendixProofDistributed}
We next present some notation and few lemmas that lead up to the main result of this section (Theorem~\ref{THM:DISTRIBUTED}). The domain of candidate atoms $\cX$ is split into $\{\cX_j, j \in [s]\}$ over $s$ machines, where machine $j$ runs WKH on $\cX_j$. Let $\sfG_j$ be the $k$-sized solution returned by running Algorithm~\ref{algo:wkh} on $\cX_j$, i.e., $\sfG_j = \text{WKH}(\cX_j, k)$. Note that each $\cX_j$ induces a partition onto the optimal $r$-sized solution $\sfS_r^\star$ as follows ($r=1$ for this theorem): 
	\begin{eqnarray*}
		\sfT_j := \{ x \in \sfS^\star_1: x \notin \wkh(\cX_j \cup x, k)\},\\
		\sfT^c_j := \{ x \in \sfS^\star_1: x \in \wkh(\cX_j \cup x, k)\}.
	\end{eqnarray*} 
In other words, $\sfT_j = \sfS^\star_1$ if the $j^\text{th}$ machine running WKH on $\cX_j \cup \sfS^\star_1$ will not select it as among its output, and it is empty otherwise, since $\sfS_1^\star$ is a singleton. We re-use the definition of $l(\cdot)$ used in Appendix~\ref{sec:appendixWKHLinearProof}.

Before moving to the proof of the main theorem, we prove two prerequisites. Recall $\sfG_j$ is the set of iterates selected by machine $j$. In this mini-result, we lower bound the expected improvement in the loss at the aggregator machine.

\begin{lemma}
	\label{lem:aggregatorGood}
	For the aggregator machine that runs $\wkh$ over $\cup_j \sfG_j$ (step 6 of Algorithm~\ref{algo:distributed}), we have, at selection of next sample point $\bx_i$ after having selected $\sfS$, $\exists$ machine $j$ such that 
	\[ 
	\bbE [l( \sfS \cup \{\bx_i\}   ) - l (\sfS)] \geq \frac{m_\omega}{ M_\Omega}  \bbE    \left( l(\sfT^c_j) - l(\sfS) \right)  .  
	\]
\end{lemma}

\begin{proof}
	The expectation is over the random split of $\cX$ into $\cX_j $ for $j \in [s]$. We denote $\sfT^c_j$ to be the complement of $\sfT_j$.
	Then, we have that
\begin{equation*}
\begin{split}
& \hspace{-5mm} \bbE [l( \sfS \cup \{\bx_i\}   ) - l (\sfS)]  \\& \geq \bbE [\frac{1}{2 M_\Omega} \gamma_\sfS ] \\&
 \geq  \frac{1}{2 M_\Omega} \sum_{\bx \in \sfS^\star_1} \bbP(\bx \in \cup_j \sfG_j) \bbE  \langle\phi(\bx), \nabla l(\sfS) \rangle_k^2\\&
 =  \frac{1}{2 s M_\Omega} \sum_{\bx \in \sfS^\star_1}  \left[ \sum_{b=1}^s \bbP(\bx \in \sfT^c_b) \right] \bbE  \langle\phi(\bx), \nabla l(\sfS) \rangle_k^2\\&
  =  \frac{1}{2 s M_\Omega}  \sum_{b=1}^s \sum_{\bx \in \sfT_b^c} \bbE  \langle\phi(\bx), \nabla l(\sfS) \rangle_k^2\\&
 \geq  \frac{m_\omega}{ s M_\Omega}  \sum_{b=1}^s  \bbE   \left( l( \sfS \cup \sfT^c_b) - l(\sfS) \right)\\&
 \geq  \frac{m_\omega}{ s M_\Omega}  \sum_{b=1}^s  \bbE   \left( l(\sfT^c_b) - l(\sfS) \right) \\&
 \geq \frac{m_\omega}{ M_\Omega} \min_{b \in [s]}  \bbE   \left( l(\sfT^c_b) - l(\sfS) \right)  .
 \end{split} 
\end{equation*}	 
The equality in step 3 above is because of Lemma~\ref{lem:probabilityOfPartitoning}.
\end{proof}

In the following lemma, we lower bound the greedy improvement in the loss on each machine.

\begin{lemma}
	\label{lem:individualgood} For any individual worker machine $j$ running local $\wkh$, if $\sfS$ is the set of $(i-1)$ iterates already chosen, then at the selection of next sample point $\bx_i$,
	 $ l( \sfS \cup \{\bx_i\}   )  \geq   \left( l( \sfT_j) - l(\sfS) \right)$.
\end{lemma}
\begin{proof}
We proceed as in proof of Theorem~\ref{THM:WKHLINEAR} in Appendix~\ref{sec:appendixWKHLinearProof}. From~\eqref{eq:lowerboundStep1}, we have,

\begin{align*}
l( \sfS \cup \{\bx\}   ) - l (\sfS) & \geq  \frac{1}{2 M_\Omega} \gamma_\sfS \\
& \geq  \frac{1}{2 M_\Omega} \sum_{\bx_j \in \sfT_j} \langle\phi(\bx_j), \nabla l(\sfS) \rangle_k^2 \\
& \geq \frac{ m_\omega}{  M_\Omega} \left( l( \sfS \cup \sfT_j) - l(\sfS) \right)\\
& \geq \frac{ m_\omega}{  M_\Omega} \left( l( \sfT_j) - l(\sfS) \right).
\end{align*}
\end{proof}

We are now ready to prove Theorem~\ref{THM:DISTRIBUTED}. 
\begin{proof}[Proof of Theorem~\ref{THM:DISTRIBUTED}]
	If, for a random split of $\cX$, for any $j \in [s]$, $\sfT_j = \sfS_1^\star$, then the given rate follows from Lemma~\ref{lem:individualgood}, after following the straightforward steps covered in proof of Theorem~\ref{THM:WKHLINEAR} for proving the rate from the given condition on~$l(\cdot)$. 
	On the other hand, if none of  $j \in [s]$, $\sfT_j = \sfS_1^\star$, then $\forall j \in [s], \sfT_j = \emptyset \implies \sfT_j^c = \sfS_1^\star$. In this case, the given rate follows from Lemma~\ref{lem:aggregatorGood}.
\end{proof}

	Finally, here is the statement and proof of an auxiliary lemma that was used above.
	
	\begin{lemma}
		\label{lem:probabilityOfPartitoning}
		For any $x \in  \cX,  \bbP (x \in \cup_j  \sfG_j) = \frac{1}{s} \sum_j \bbP(x \in \sfT_j^c)$.
	\end{lemma}
	\begin{proof}
	We have
	
	\begin{align*}
	& \hspace{-5mm} \bbP (x \in \cup_j  \sfG_j) \\
	&=  \sum_j \bbP ( x \in \cX_j \cap x \in \wkh(\cX_j,k)) \\
	&= \sum_j \bbP(x \in \cX_j) \bbP(x \in \wkh(\cX_j,k) | x \in \cX_j)\\
	&= \sum_j \bbP(x \in \cX_j) \bbP(x \in \sfT^c_j)\\
	& = \frac{1}{s}\bbP(x \in \sfT^c_j).
	\end{align*}
	
\end{proof}

\end{document}